\documentclass[lettersize,journal]{IEEEtran}
\usepackage{amsmath,amsfonts,amssymb,mathrsfs}
\usepackage{algorithmic}
\usepackage{algorithm}
\usepackage{array}
\usepackage[caption=false,font=normalsize,labelfont=sf,textfont=sf]{subfig}
\usepackage{textcomp}
\usepackage{stfloats}
\usepackage{url}
\usepackage{verbatim}
\usepackage{graphicx}
\usepackage{cite}
\usepackage{calligra}
\usepackage{cuted}
\usepackage{blindtext}
\usepackage[colorlinks,
            linkcolor=blue,
            anchorcolor=blue,
            citecolor=blue
            ]{hyperref}

\graphicspath{{Fig/}}
\newtheorem{theorem}{Theorem}
\newtheorem{definition}{Definition}
\newtheorem{proof}{Proof}
\numberwithin{equation}{section}

\begin{document}

\title{RM-Dijkstra: A surface optimal path planning algorithm based on Riemannian metric}

\author{Yu Zhang,~\IEEEmembership{}Xiao-Song Yang~\IEEEmembership{}
\thanks{Yu Zhang is with the School of Mathematics and Statistics, Huazhong University of Science and Technology, Wuhan 430074, China (email: zhangyu348@hust.edu.cn)}
\thanks{Xiao-Song Yang is with the School of Mathematics and Statistics, Huazhong University of Science and Technology, Wuhan 430074, China, and also with Hubei Key Laboratory of Engineering Modeling and Scientific Computing, Huazhong University of Science and Technology, Wuhan 430074, China (email: yangxs@hust.edu.cn)}}

\markboth{}%
{Shell \MakeLowercase{\textit{et al.}}: A Sample Article Using IEEEtran.cls for IEEE Journals}


\maketitle
\pagestyle{empty}  
\thispagestyle{empty} 

\begin{abstract}
The Dijkstra algorithm is a classic path planning method, which operates in a discrete graph space to determine the shortest path from a specified source point to a target node or all other nodes based on non-negative edge weights. Numerous studies have focused on the Dijkstra algorithm due to its potential application. However, its application in surface path planning for mobile robots remains largely unexplored. In this letter, a surface optimal path planning algorithm called RM-Dijkstra is proposed, which is based on Riemannian metric model. By constructing a new Riemannian metric on the 2D projection plane, the surface optimal path planning problem is therefore transformed into a geometric problem on the 2D plane with new Riemannian metric. Induced by the standard Euclidean metric on surface, the constructed new metric reflects environmental information of the robot and ensures that the projection map is an isometric immersion. By conducting a series of simulation tests, the experimental results demonstrate that the RM-Dijkstra algorithm not only effectively solves the optimal path planning problem on surfaces, but also outperforms traditional path planning algorithms in terms of path accuracy and smoothness, particularly in complex scenarios.        
\end{abstract}


\section{Introduction}
\label{sec:introduction}

\IEEEPARstart{P}{ath} planning has long been a central topic in robotics research\cite{1motionplanning}\cite{2handbookofrobotics}\cite{3ral} and autonomous driving\cite{4autonomousdriving}. It aims to find a collision-free route from the start point to the target destination, avoiding obstacles along the way. Path-planning algorithms are now extensively applied across diverse fields, including robot navigation\cite{5robotnavigation}, logistics and warehousing\cite{6logisticsandwarehousing}, urban planning and traffic management\cite{7urbanplanningandtrafficmanagement}, game development\cite{8gamedevelopment} and so on.
Path planning algorithms can be classified according to their working principles or core ideas\cite{9zongshu}. The most widely used ones are graph-based path planning algorithms and sampling-based path planning algorithms.

Graph-based path planning algorithms discretize the environment into nodes and edges, using graph search methods to find the shortest path from the start point to the goal. Common algorithms, such as Dijkstra and A*, are used for unweighted and weighted graphs with heuristic functions, respectively, and are widely applied in path planning problems in 2D or 3D spaces.
The Dijkstra algorithm\cite{10Dijkstra} is a shortest-path algorithm that systematically explores a graph by selecting the node with smallest tentative distance, ensuring an optimal solution for graphs with non-negative edge weights. 
Due to its strong theoretical foundation, high efficiency, and simplicity of implementation, the Dijkstra algorithm has become one of the most recognized algorithms in computer science and mathematics. It is a cornerstone of graph theory, and many more complex algorithms are built upon its principles or refine it (such as A* and Bellman-Ford) to tackle a broader range of path planning problems.
In the latest breakthrough research article\cite{12UniversalOptimalityofDijkstra}, this classic Dijkstra algorithm has been proven to be universally optimal, which means that no matter how complex the graph structure is, it can achieve theoretically optimal performance even in the worst case!
The A* algorithm\cite{11Astar} improves upon Dijkstra by incorporating a heuristic function to estimate the cost to the goal, thereby improving path-planning efficiency by prioritizing nodes that are more likely to lead to the goal, while still guaranteeing optimality when using an admissible heuristic.

However, graph-based path planning algorithms are mainly used to solve the shortest path problem in discrete space, and cannot be directly applied to path planning problems in continuous space\cite{14CDT-Dijkstra}, such as surfaces. 
Although the traditional Dijkstra algorithm can be applied by discretizing the continuous space into a grid\cite{13}, discretization will result in a loss of computational accuracy, especially in high-dimensional space, where the number of grids may be very large, resulting in high computational complexity.
In particular, when we use the traditional Dijkstra algorithm to solve the path planning problem on a two-dimensional surface in three-dimensional space, the most serious disadvantage is usually the inability to naturally deal with the constraints of the surface, that is, the traditional Dijkstra algorithm cannot consider the geometric characteristics and continuity of the surface.
This is a fundamental problem because the quality and feasibility of path planning depends heavily on the algorithm's understanding of surface morphology. If the surface is not adequately modeled and processed, the path generated by the traditional Dijkstra algorithm may be wrong or not feasible.
On a two-dimensional surface in three-dimensional space, the shortest path should generally not be a straight line or a connection of discrete points, but a smooth curve along the surface. 
Since the geometric properties of the surface (such as bending and curvature) cannot be effectively represented by Euclidean distance, traditional Dijkstra algorithm fails to take into account the intrinsic characteristics of these surfaces, resulting in calculated paths that may deviate from the surface or produce unnatural distortions on the surface.

In recent years, the concept of Riemannian metric has gradually gained attention in the field of path planning\cite{15liu2015robotic}\cite{16kularatne2016time}. 
Riemann metric is a metric structure defined on a manifold, which gives the concepts of distance and angle to the tangent space on the manifold in the form of an inner product\cite{17riemanniangeometry}\cite{18riemannian}. It provides a way to measure the length of a curve, the angle between vectors, and the curvature of a surface, thereby generalizing the geometric concepts in Euclidean space to curved space.
With the Riemannian metric, it is possible to calculate the distance between two points on a manifold and define geodesics as the shortest paths that follow the curvature and geometric properties of the manifold. Riemann metric is applicable not only to flat Euclidean space, but also to non-Euclidean space with curvature, and is widely used in path planning problems\cite{19} in complex geometric environments.

In this letter, we apply the idea of Riemannian metric to a graph-based path planning algorithm to solve the optimal path planning problem on surfaces. By constructing a new Riemannian metric on the 2D projection plane, the surface optimal path planning problem is therefore transformed into a geometric problem on the 2D projection plane with new Riemannian metric. The projection map is an isometric immersion, which ensures that the length of a curve on surface is equal to the length of a new curve on its 2D projection plane in the new metric sense. We propose a surface optimal path planning algorithm called RM-Dijkstra in this letter, which can accurately plan shortest smooth path from one point to target point on surfaces. A series of simulation experiments in scenarios with different surface curvature are presented to verify the validity of RM-Dijkstra. We also conduct comparative experiments with original Dijkstra algorithm using Euclidean distance and original A* algorithm using Euclidean distance to strengthen our results. 
Since the path retrieved by RM-Dijkstra on the surface is the preimage of the polyline path on 2D projection plane (in the new metric sense) under the projection map, the path output by the RM-Dijkstra algorithm has a natural piecewise smooth property and is closely attached to the surface, as shown in section \ref{sec:Simulation}.
The experimental results show that the RM-Dijkstra algorithm effectively solves the optimal path planning problem on surfaces and significantly outperforms traditional path planning algorithms, particularly in terms of path accuracy and piecewise smoothness, especially in complex scenarios.

The rest of this letter is organized as follows. The preliminaries and methodology are presented in Section \ref{sec:Preliminaries and Methodology}. Details of the RM-Dijkstra algorithm are presented in Section\ref{sec:Algorithm}. The experimental results are shown in Section\ref{sec:Simulation}, which are followed by conclusions in Section\ref{sec:Conclusion}

\section{Preliminaries and Methodology}
\label{sec:Preliminaries and Methodology}

\subsection{Riemannian geometry}
In this subsection, we present a theoretical framework for path planning. The framework is mainly based on Riemannian Geometry as follows.

The Riemannian geometry\cite{17riemanniangeometry} is a branch of mathematics that studies the geometric properties of manifolds equipped with a Riemannian metric. A manifold is a topological space that is locally homeomorphic to Euclidean space, and surfaces are two-dimensional manifolds. Each point on a manifold has a tangent space, which consists of the tangent vectors of curves passing through that point. The tangent space provides a tool for describing the local geometry of the manifold. A Riemannian metric is an inner product structure defined on the manifold that provides an inner product on each tangent space, allowing us to define quantities such as curve length, angles, and to analyze curvature and other geometric properties. An isometric immersion is a smooth map that locally embeds a manifold into another manifold while preserving distances locally. An isometric embedding is a special case of isometric immersion, where the map is not only isometric but also injective , preserving the manifold's topological structure

\begin{definition}
The tensor $g$ is referred to as the Riemannian metric on $M$, if $g$ satisfies both symmetry and positive definiteness.

(1)symmetry: $g(u,v)=g(v,u)$, $\forall u,v\in T_{P}M $

(2)positive definiteness: $g(u,u)\ge$ 0 for all $ u\in$ $T_{p}M $, where the equal sign holds if and only if $u = 0$
\end{definition}

\begin{definition}{(isometric immersion)}
An isometric immersion between Riemannian manifolds $(M,g)$ and $(N,h)$ is a smooth mapping $f: M\to N$ such that $g=f^{\ast } h$,i.e. $h(f_{\ast }(v), f_{\ast }(w))=g(v,w)$ for all tangent vectors $v,w\in T_{p}M $ and all $p\in M$
\end{definition}

Now, we consider a height surface in $\mathbb{R}^3$, which is represented by a smooth function $x_3: \mathbb{R}^2 \to \mathbb{R}$. The surface can be parametrized as $X_h: (x_1, x_2)\mapsto (x_1, x_2, x_3(x_1, x_2))$, where $x_1$ and $x_2$ are parameters on the plane. The surface $M= X_h(\mathbb{R}^2)$  is naturally covered by the chart $X_h^{-1} $, which maps the height surface back to the plane. To determine the basis vectors of the tangent plane at each point on the surface, we compute the partial derivatives of the surface. Let $P=(x_1^P, x_2^P, x_3^P)$ be a point on the surface. 
To compute the first basis vector, we fix $x_2 = x_2^P$ and vary $x_1$ infinitesimally (e.g. $x_1 = x_1^P + \varepsilon $ , where $\varepsilon $ is an infinitesimal quantity). The corresponding change in the surface height is $ \triangle x_3=x_3(x_1^P+\varepsilon,x_2^P)-x_3 (x_1^P, x_2^P)$. 
The components of the displacement in the $(x_1, x_2, x_3)$-coordinates are then  $\triangle x_1=\varepsilon$, $\triangle x_2=0$, and $\triangle x_3= \triangle x_3$. Therefore, the first basis vector of the tangent plane is given by $ \vec{e_1}=(1,0,\frac{\partial x_3}{\partial x_1} )$. Similarly, the second basis vector is $\vec{e_2}=(0,1,\frac{\partial x_3}{\partial x_2} )$. 
With these two basis vectors, any vector in the tangent plane can be expressed as a linear combination of $\vec{e_1}$ and $\vec{e_2}$. For a vector with components  $(v_1, v_2, v_3)$ in the tangent plane, it can be written as $$(v_1, v_2, v_3)= v_1\cdot (1,0,\frac{\partial x_3}{\partial x_1} )+ v_2\cdot (0,1,\frac{\partial x_3}{\partial x_2} ) $$
It follows that the third component of the vector is $$v_3= v_1\cdot \frac{\partial x_3}{\partial x_1}+v_2\cdot \frac{\partial x_3}{\partial x_2}$$. 
This provides the full expression for any vector in the tangent plane at the point $P$ on the height surface.

\subsection{Methodology}
In this section, we propose a general framework for optimal path planning by defining a new Riemannian metric on the projection plane $\mathbb{R}^{2}=\left\{\left(x_{1}, x_{2}\right) \mid x_{1}, x_{2} \in \mathbb{R}\right\} $. This metric is induced by the standard Euclidean metric on surface embedded in $\mathbb{R}^3 $, providing a new approach to solving surface path planning problems.

The robot's motion considered in the present paper can be modeled on a 2D manifold represented by the height function $x_3 (x_1,x_2)$ in 3D space. 
Denote the two-dimensional smooth manifold $M$ as $$\begin{aligned}
\vec{s}&: U \subset\mathbb{R}^2 \rightarrow \mathbb{R}^3 \\
\vec{s}\left(x_1, x_2\right)&=\left(x_1, x_2, x_3\left(x_1, x_2\right)\right)
\end{aligned}$$
where $U$ is a open set in $\mathbb{R}^2$. Obviously, $\vec{s}\left( U \right)$ is embedded in the three-dimensional space $\mathbb{R}^3$,where the standard Euclidean metric can be denoted as $g_{ij}$. The following is about its calculation: 
$$\begin{aligned}
g_{11}&=g((1,0,0),(1,0,0))=1 \\
g_{12}&=g((1,0,0),(0,1,0))=0
\end{aligned}$$
Similarly, 
$g_{ii}=1 (i=1,2,3) $ and $g_{ij}=0 (i\ne j)$.

Finally we have  
$g=g_{i j} d x^{i} \otimes d x^j=g_{ij} d x^i d x^j =g_{11} d x^1 d x^1+g_{22} d x^2 d x^2+g_{33} d x^3 d x^3 = d x^1 d x^1+d x^2 d x^2+d x^3 d x^3 $. Choose two vectors on the tangent plane $T_{p}M $ at any point $p$ on the manifold $\vec{s}(U)$ as follows:
$$\begin{aligned}
&\vec{c_1}=\left(1,0, \frac{\partial x_3}{\partial x_1}\right)  , 
&\vec{c_2}=\left(0,1, \frac{\partial x_3}{\partial x_2}\right) 
\end{aligned}$$
Now we can construct a new Riemannian metric $h=h_{i j} d x^i d x^j=h_{11} d x^1 d x^1 + h_{12} d x^1 d x^2 + h_{21} d x^2 d x^1 + h_{22} d x^2 d x^2$ on the projection plane $\mathbb{R}^2 $.
Let 
$$
\begin{aligned}
h_{11}&=g(\vec{c_1}, \vec{c_1})=1+\left(\frac{\partial x_3}{\partial x_1}\right)^2\\
h_{12}&=h_{21}=g( \vec{c_1}, \vec{c_2} )=\frac{\partial x_{3}}{\partial x_{1}} \cdot \frac{\partial x_3}{\partial x_2}\\
h_{22}&=g( \vec{c_2}, \vec{c_2} )=1+\left(\frac{\partial x_3}{\partial x_2}\right)^2
\end{aligned}
$$

\noindent Thus the Metric matrix $[h_{ij}]$ can be written as 
$$
\large \left[h_{ij}\right]=\begin{bmatrix}
    1+\left(\frac{\partial x_3}{\partial x_1}\right)^2 & \frac{\partial x_3}{\partial x_1} \cdot \frac{\partial x_3}{\partial x_2}\\
   \frac{\partial x_3}{\partial x_1} \cdot \frac{\partial x_3}{\partial x_2} & 1+\left(\frac{\partial x_3}{\partial x_2}\right)^2 
\end{bmatrix}
$$

Furthermore, the inverse matrix can be calculated by
$\left[h_{i j}\right]^{-1} \triangleq\left[h^{k l}\right]$

$${\large \left[h^{k l}\right]=\frac{1}{d}\begin{bmatrix}
1+\left(\frac{\partial x_3}{\partial x_2}\right)^2&-\frac{\partial x_3}{\partial x_1} \cdot \frac{\partial x_3}{\partial x_2} \\
-\frac{\partial x_3}{\partial x_1} \cdot \frac{\partial x_3}{\partial x_2} & 1+\left(\frac{\partial x_3}{\partial x_1}\right)^2
\end{bmatrix}
} $$
here
$$
\begin{aligned}
  d= \det(h_{ij})=&\left[1+\left(\frac{\partial x_3}{\partial x_1}\right)^2\right] \cdot \left[1+\left(\frac{\partial x_3}{\partial x_2}\right)^2\right] \\
  &-\left( \frac{\partial x_3}{\partial x_1} \cdot \frac{\partial x_3}{\partial x_2}\right)^{2}  
\end{aligned}  
$$

We also point out that the projection from this two-dimensional smooth manifold $\vec{s}$ to projection plane $\mathbb{R}^2$ is an isometric map.
\begin{theorem}\label{thm2}
The two-dimensional smooth manifold $M$ in $\mathbb{R}^3$ with a metric induced by the standard Euclidean metric $g=d x^1 d x^1+d x^2 d x^2+d x^3 d x^3$ is a Riemannian manifold. If its projection plane $\mathbb{R}^2=\left(x_{1}, x_{2}\right)$ is equipped with a new Riemannian metric $h=h_{i j} d x^i d x^j$ as above, then the Projection mapping $f: M \rightarrow \mathbb{R}^2$ is an isometric immersion(embedding).
\end{theorem}
\begin{proof}
    We showed in Definition 2 that a diffeomorphism is isometric if it satisfies $g=f^{\ast } h$. Hence, it remains to be shown that $h(f_{\ast }(v), f_{\ast }(w))=g(v,w)$ for all tangent vectors $v,w\in T_{p}M $ and all $p\in M$. Let $v=(v_1,v_2,v_3),w=(w_1,w_2,w_3)$. The left side equals (II.1)

\begin{align}
    &h(f_{\ast }(v), f_{\ast }(w))=  h\left(\left(v_{1}, v_{2}\right),\left(w_1, w_2\right)\right) \notag\\
    &=  h_{11} v_1 w_1+h_{12} v_1 w_2+h_{21} v_2 w_1+h_{22} v_2 w_2 \notag\\
    &=  v_1 w_1\left[1+\left(\frac{\partial x_3}{\partial x_1}\right)^2\right] + v_2 w_2 \left[1+\left(\frac{\partial x_3}{\partial x_2}\right)^2\right]\notag\\
    &+\left(v_1 w_2+v_2 w_1\right) \left( \frac{\partial x_3}{\partial x_1}\cdot\frac{\partial x_3}{\partial x_2}\right) 
\end{align}

Since 
$\left(v_1, v_2, v_3\right)=\left(v_1, 0, v_1 \frac{\partial x_3}{\partial x_1} \right)+\left(0, v_2, v_2\frac{\partial x_3}{\partial x_2}\right)$, then $v_3$ has an expression $v_3=v_1\frac{\partial x_3}{\partial x_1}+v_2 \frac{\partial x_3}{\partial x_2}$. Similarly, $w_3=w_1\frac{\partial x_3}{\partial x_1}+w_2 \frac{\partial x_3}{\partial x_2}$. Thus the right side equals (II.2)

\begin{align}
& g(v,w)= v_1 w_1+v_2 w_2+v_3 w_3 \notag\\
& = v_1w_1\left[1+\left(\frac{\partial x_3}{\partial x_1}\right)^2\right]  
+ v_2 w_2\left[1+\left(\frac{\partial x_3}{\partial x_2}\right)^2\right]\notag\\
& +\left(v_1 w_2 + v_2 w_1\right)\left( \frac{\partial x_3}{\partial x_1}\cdot\frac{\partial x_3}{\partial x_2}\right)
\end{align} 

We see the left side equals the right side and this completes the proof.
\end{proof}

Using the proposed Riemannian metric $h=h_{i j} dx^i dx^j$ , the length (cost function) of a smooth curve $\gamma:[a, b] \rightarrow \mathbb{R}^2$ on projection plane $\mathbb{R}^2 $ is obtained by the integral over inner products on $\mathbb{R}^2 $ as shown in (II.3).
\begin{figure*}[t]
\begin{align}
L(\gamma) = &\int_a^b \left|\gamma^{\prime}(t)\right| dt=\int_a^b \sqrt{h\left(\gamma^{\prime}(t), \gamma^{\prime}(t)\right)} dt \notag\\
= &\int_a^b \sqrt{h\left(\left(x_1^{\prime}(t), x_2^{\prime}(t)\right),\left(x_1^{\prime}(t), x_2^{\prime}(t)\right)\right)} dt \notag\\
= &\int_a^b \sqrt{h_{11}\left[x_1^{\prime}(t)\right]^2+h_{12}\left[x_1^{\prime}(t)\cdot  x_2^{\prime}(t)\right]+h_{21}\left[x_2^{\prime}(t) \cdot x_1^{\prime}(t)\right]+h_{22}\left[x_2^{\prime}(t)\right]^2} dt \notag\\
= &\int_a^b \sqrt{h_{11}\left[x_1^{\prime}(t)\right]^2+2 h_{12}\left[x_1^{\prime}(t) \cdot x_2^{\prime}(t)\right]+h_{22}\left[x_2^{\prime}(t)\right]^2} dt \notag\\
= &\int_a^b \sqrt{\left[1+\left(\frac{\partial x_3}{\partial x_1}\right)^2\right] \cdot\left[x_1^{\prime}(t)\right]^2+2\left( \frac{\partial x_3}{\partial x_1}\cdot\frac{\partial x_3}{\partial x_2}\right) \cdot\left[x_1^{\prime}(t)\cdot x_2^{\prime}(t)\right]
+\left[1+\left(\frac{\partial x_3}{\partial x_2}\right)^2\right] \cdot\left[x_2^{\prime}(t)\right]^2}dt\\
\rule[-1pt]{1.2cm}{0.05em} &\rule[-1pt]{16.6cm}{0.05em} \notag
\end{align}

\end{figure*}

\section{Algorithm}
\label{sec:Algorithm}

In this section, we present an optimal path planning algorithm RM-Dijkstra (Dijkstra based on Riemannian metric) as an example of the application of Riemannian metric in path planning. RM-Dijkstra consists of two stages: Estimation of straight-line distance stage and Path searching stage.

As mentioned in the preliminaries, we define a new Riemannian metric on the two-dimensional plane, which is not a traditional Euclidean metric.
Therefore, a new method is needed to compute the distance between two points on the two-dimensional plane with this new Riemannian metric, as shown in Algorithm 2.
In the pseudocode, a parametric curve from the start point to the end point of a straight line segment on the two-dimensional plane is defined.
$$point(t)=start+t*(end-start)$$
Here t is the scale parameter of the line. t=0 is start point and t=1 is end point.
The expression by calculating the length of the path element is 
$$Integrand(t)=\sqrt{h_{11} \cdot \left ( \Delta x_1 \right ) ^{2} +h_{22} \cdot \left ( \Delta x_2 \right ) ^{2} +2 h_{12} \Delta x_1 \Delta x_2}$$
where $\Delta x_1= end[0]-start[0], \Delta x_2= end[1]-start[1]$, and $h_{ij}$ is the constructed Riemannian metric tensor.
Since the length of the entire line segment is obtained by integrating the length of the path microelement, we invoke the Gauss-Legendre numerical integration method to approximate the integral value, as shown in Algorithm 1.
Gauss-legendre numerical integration method selects Legendre polynomial roots in the interval as sampling points, and combines the corresponding weights to sum the function, so as to approximate the integral value. This method maximizes the approximation accuracy and achieves highly accurate integration with fewer sampling points.
$$\int_{a}^{b} f(x) d x \approx \sum_{i=1}^{n} w_{i}^{\text {mapped }} \cdot f\left(x_{i}^{\text {mapped }}\right)$$

The main part of RM-Dijkstra is shown in Algorithm 3.
Here we take random points in the two-dimensional plane with new metric, and use Algorithm 2 to compute the adjacency matrix of the graph. The rest of the steps are similar to the traditional Dijkstra algorithm.




\begin{algorithm}
\caption{Gauss-Legendre Integration}
\begin{algorithmic}
\STATE \textbf{Input:} 

\STATE \quad $func \gets$ Function to be integrated;
\STATE \quad $a,b \gets$ Integration bounds;
\STATE \quad $n \gets$ Number of Gauss points;
\STATE \textbf{GaussLegendreIntegrate:} 
\STATE \quad $nodes, weights = leggauss(n)$
\STATE \quad $m\_nodes = 0.5 \cdot (b - a) \cdot nodes + 0.5 \cdot (b + a)$ 
\STATE \quad $m\_weights = 0.5 \cdot (b - a) \cdot weights$ 
\STATE \textbf{return} $\sum_{i=1}^{n} \left(m\_weights[i] \cdot func(m\_nodes[i])\right)$ 
\end{algorithmic}
\end{algorithm}

\begin{algorithm}
\caption{Estimate Line Distance}
\begin{algorithmic}
\STATE \textbf{Input:} 
\STATE \quad $start \gets$ Starting point of the line;
\STATE \quad $end \gets$ Ending point of the line;
\STATE \quad $samples \gets$ Number of Gauss-Legendre points;
\STATE \quad $h_{11} \gets 1 + \left(\frac{\partial S}{\partial x_1}\right)^2, \quad h_{12} \gets \frac{\partial S}{\partial x_1} \frac{\partial S}{\partial x_2},$ 
\STATE \quad $\quad h_{22} \gets 1 + \left(\frac{\partial S}{\partial x_2}\right)^2$

\STATE \textbf{def LineEstimateDist:} 
    \STATE \quad \textbf{def Integrand(t):}
        \STATE \quad \textbf{return}  $ \sqrt{h_{11} \Delta x_1^2 + h_{22} \Delta x_2^2 + 2h_{12} \Delta x_1 \Delta x_2} $
\STATE \textbf{return}  $GaussLegendreIntegrate(Integrand,0,1,samples)$
\end{algorithmic}
\end{algorithm}

\begin{algorithm}
\caption{RM-Dijkstra with Riemannian Metric on a Surface}
\begin{algorithmic}
\STATE \textbf{Input:} 
\STATE \quad $n \gets$ Number of nodes in the graph;
\STATE \quad $start, end \gets$ Starting and ending points for Dijkstra;
\STATE \quad \texttt{graph} $\gets$ Initialize a zero matrix of size $n \times n$;
\STATE $points = random(n)$
\FOR{$i = 0$ to $n-1$}
    \FOR{$j = i+1$ to $n-1$} 
        \STATE $graph[i][j] = LineEstimateDist(points[i], points[j]) $
        \STATE $graph[j][i] = graph[i][j]$
    \ENDFOR
\ENDFOR

\STATE
\textbf{def Dijkstra:} 
\STATE \textbf{Initialization:}
\STATE \quad $distances[start] = 0$;
\STATE \quad $distances[i] = \infty$ for all other $i$;
\STATE \quad $visited =$ False for all nodes;
\WHILE{not all nodes visited}
    \STATE $current\_node = $ node with minimum $distances[]$
    \STATE  $visited[current] = True$
    \FOR{each $neighbor$ of $current\_node$}
        \IF{$graph[current\_node][neighbor] \neq 0$ and $neighbor$ is not visited}
            \STATE $new\_distance = distances[current\_node] + graph[current\_node][neighbor]$
            \IF{$new\_distance < distances[neighbor]$}
                \STATE $distances[neighbor] = new\_distance$
                \STATE $previous[neighbor] = current\_node$
            \ENDIF
        \ENDIF
    \ENDFOR
\ENDWHILE
\STATE \textbf{return}  $Dijkstra(graph, start, end, points)$

\end{algorithmic}
\end{algorithm}


\section{Simulation}
\label{sec:Simulation}

In this section, several simulation experiments were conducted to demonstrate the effectiveness and efficiency of RM-Dijkstra algorithm. 
The computer we use is a MacBook. The MacBook setup includes an Apple M3 chip (8-core CPU and 10-core GPU) with 16 GB of unified memory. For the simulation, we use Python on macOS Sonoma.

Comparison experiments with traditional graph search-based algorithms are also included, such as traditional Dijkstra algorithm and traditional Astar algorithm using original Euclidean distance.
In experiments, the working space of these algorithms is uniformly taken as $[-1,11]\times [-1,11]$, including start point $(0,0)$ and end point $(10,10)$. 
By constructing a height map with fluctuating characteristics, the surface function can be used to simulate mountain topography. In this study, we took advantage of the properties of the Gaussian function to create local peaks and then build larger terrain by superimposing multiple such peaks.
For the two-dimensional case, the Gaussian peak function is expressed as follows
$$f(x, y)=A \cdot \exp \left(-\frac{\left(x-x_{0}\right)^{2}+\left(y-y_{0}\right)^{2}}{2 \sigma^{2}}\right)$$
where $\left ( x_{0},y_{0} \right ) $ represents the central position of the Gaussian peak, and $\sigma$ controls the expansion degree of the peak.

\begin{figure}
   \centering
   \subfloat[]{
    \includegraphics[width=0.5\linewidth]{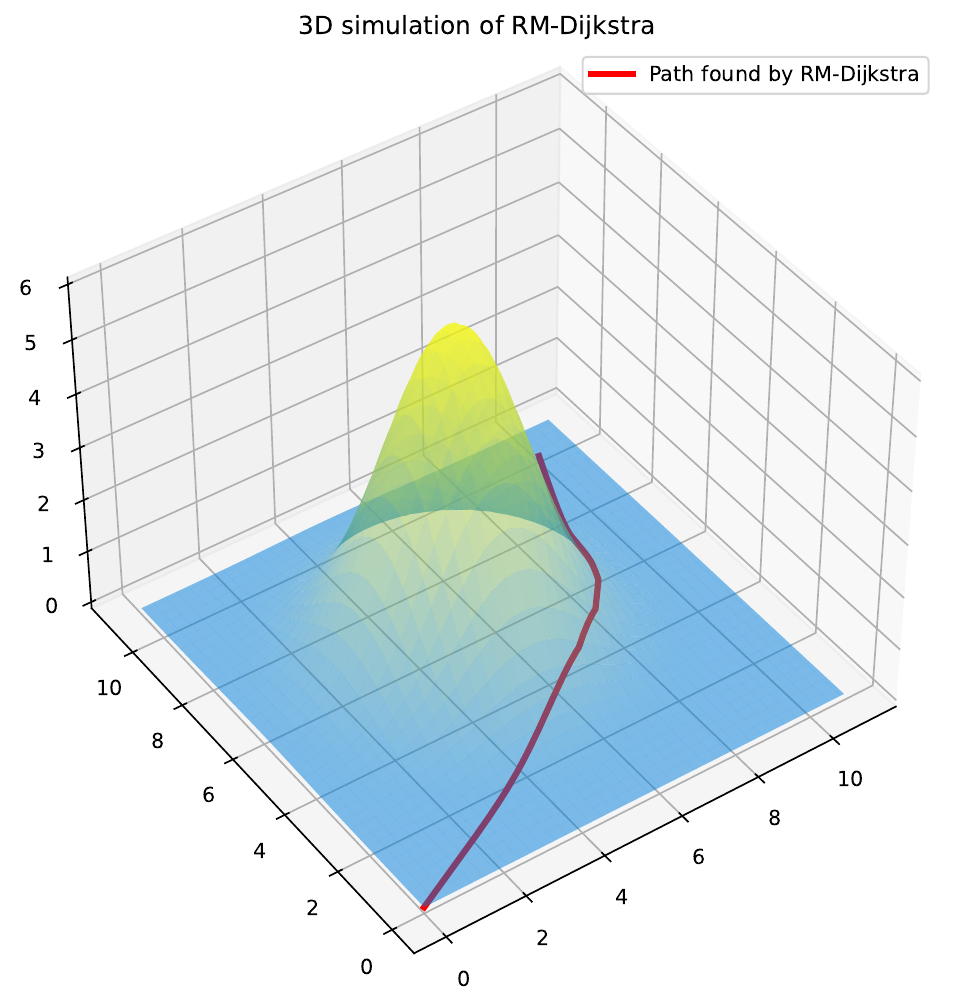}
    \label{3d1-1}
    }
   \subfloat[]{
      \includegraphics[width=0.5\linewidth]{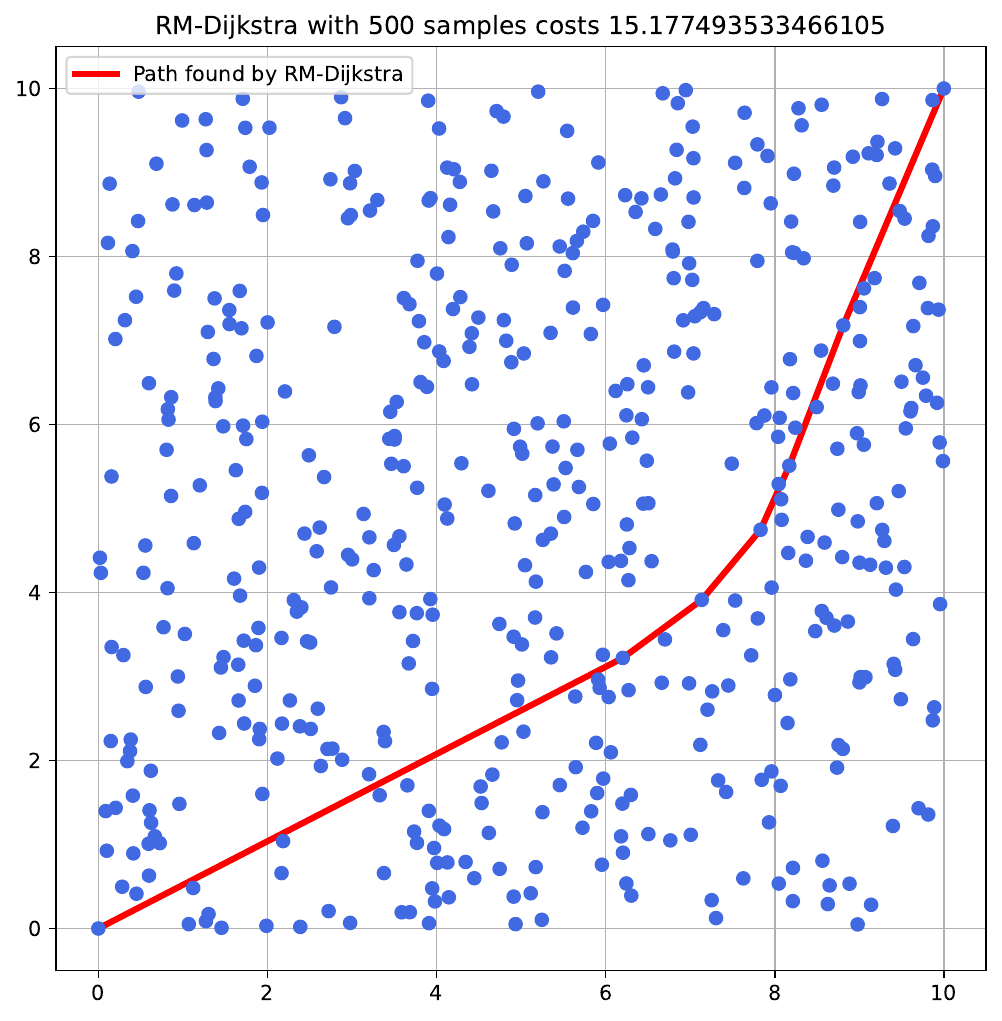}
      \label{3d1-2}
    }\\
    \subfloat[]{
      \includegraphics[width=0.5\linewidth]{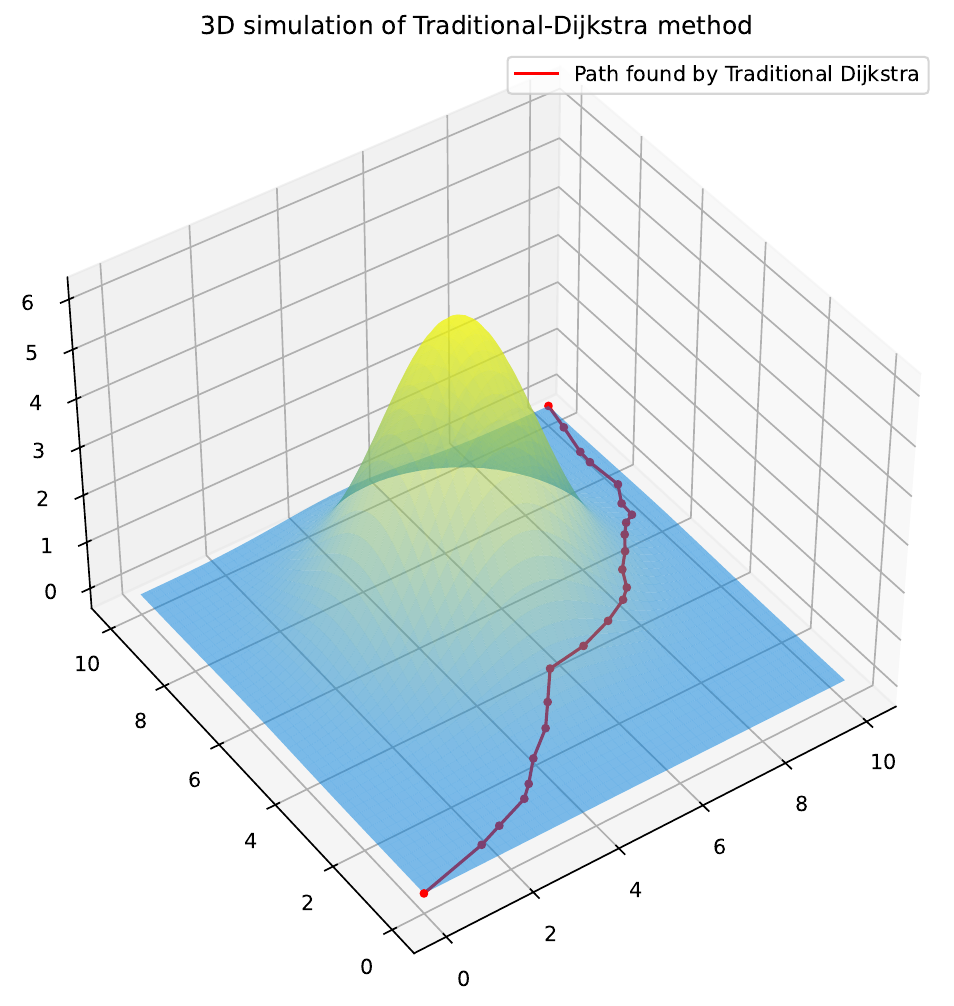}
      \label{3d1-3}
    }
    \subfloat[]{
      \includegraphics[width=0.5\linewidth]{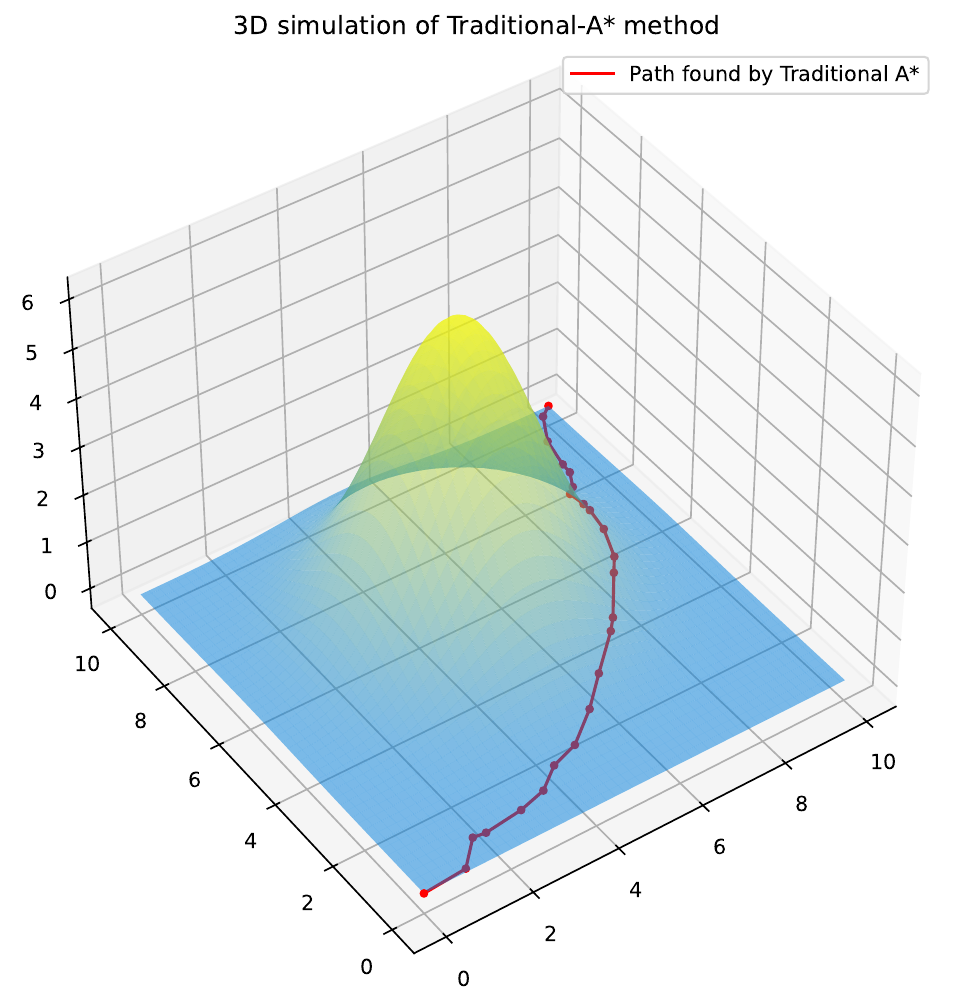}
      \label{3d1-4}
    }
   \caption{Simulation in 3-D scenario (one peak). (b) is obtained by RM-Dijkstra taking 500 random sampling points. (a) is the preimage of the path in (b) under the projection map. (c) is obtained by original Dijkstra algorithm using Euclidean distance. (d) is obtained by original A* algorithm using Euclidean distance}
   \label{onepeak}
\end{figure}

First consider a relatively simple scenario, as shown in Figure \ref{onepeak}, we choose to use the Gaussian function $x_3 = 6\cdot e^{ -\frac{1}{5}\left [(x_1-5)^2+(x_2-6)^2 \right ] } $ to simulate a one-peak surface in three-dimensional space.
Since the RM-Dijkstra algorithm transformed the path planning problem on the surface into a geometry problem on the plane with a new Riemannian metric, we called the RM-Dijkstra algorithm for path planning in the planar workspace, as shown in Figure \ref{3d1-2}.
The length of the path retrieved by RM-Dijkstra on the plane by randomly selecting 500 points is 15.177493. This path is obtained by connecting some line segments from end to end, that is, a broken line.
Mapping this polyline path on the plane onto the single-peak surface through the inverse mapping of projection mapping will result in a piecewise smooth path on the single-peak surface, as shown in Figure \ref{3d1-1}.
Since this path (in Figure \ref{3d1-1}) is the preimage of the broken line (in Figure \ref{3d1-2}) under the projection mapping, the red path finally formed in Figure \ref{3d1-1} is obtained by connecting some smooth curve segments on the single-peak surface successively through various nodes. In other words, The path output by RM-Dijkstra algorithm on a surface in three-dimensional space is naturally piecewise smooth and all the points that constitute the path are points on the surface.

In order to judge the advantages and disadvantages of RM-Dijkstra algorithm and increase its persuasiveness, we select traditional Dijkstra algorithm and traditional A* algorithm as comparative experiments, as shown in Figures \ref{3d1-3} and \ref{3d1-4}.
The comparative experiments both use the classical Euclidean metric of three-dimensional space to calculate the distance, and keep the number of random sampling points consistent with RM-Dijkstra , which is 500 random points.
The path length retrieved by the traditional Dijkstra algorithm is 15.870241, and the path length retrieved by the traditional A* algorithm is 15.766383.
By observing the trend and shape of these paths, it is obvious that the traditional graph search-based algorithms are not as accurate and smooth as the RM-Dijkstra algorithm when solving the problem of surface path planning.As the complexity of the surface increases, the advantages of RM-Dijkstra in these aspects will become more obvious, which will be introduced in detail in the following two sets of experimental scenarios.

The traditional Dijkstra algorithm has great difficulties in dealing with surface constraints, mainly because it is based on graph traversal and distance calculation. However, the geometric characteristics of the surface (such as bending, curvature, etc.) cannot be effectively expressed by simple Euclidean distance.
Since the traditional Dijkstra algorithm does not take into account the geometric characteristics of the surface itself, the calculated path usually deviates from the surface or distorts unnaturally on the surface.
On a two-dimensional surface in three-dimensional space, the shortest path is usually not a simple straight line or discrete connected points, but a continuous curve extending along the surface. The discretized graph structure of the traditional Dijkstra algorithm can not deal with this continuity effectively, so the path results deviate from the requirements of the actual surface.

\begin{figure}
   \centering
   \subfloat[]{
    \includegraphics[width=0.5\linewidth]{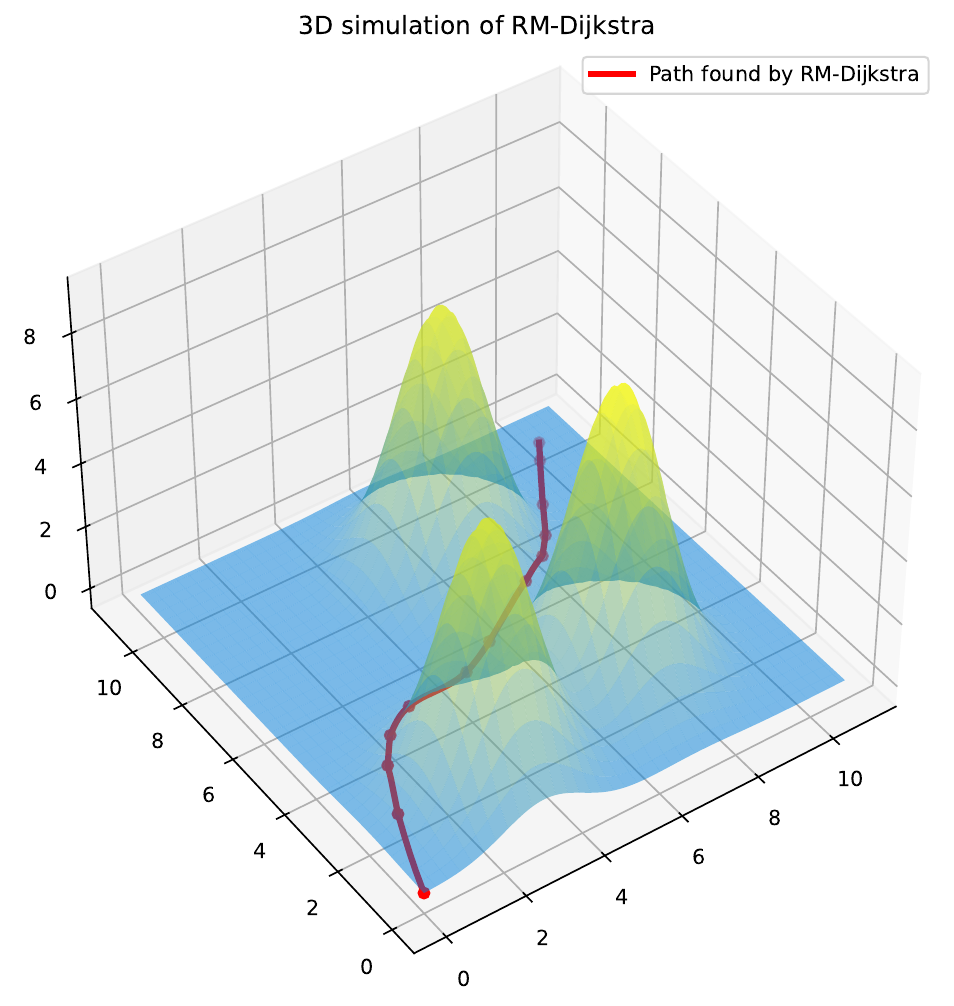}
    \label{3d3-1}
    }
   \subfloat[]{
      \includegraphics[width=0.47\linewidth]{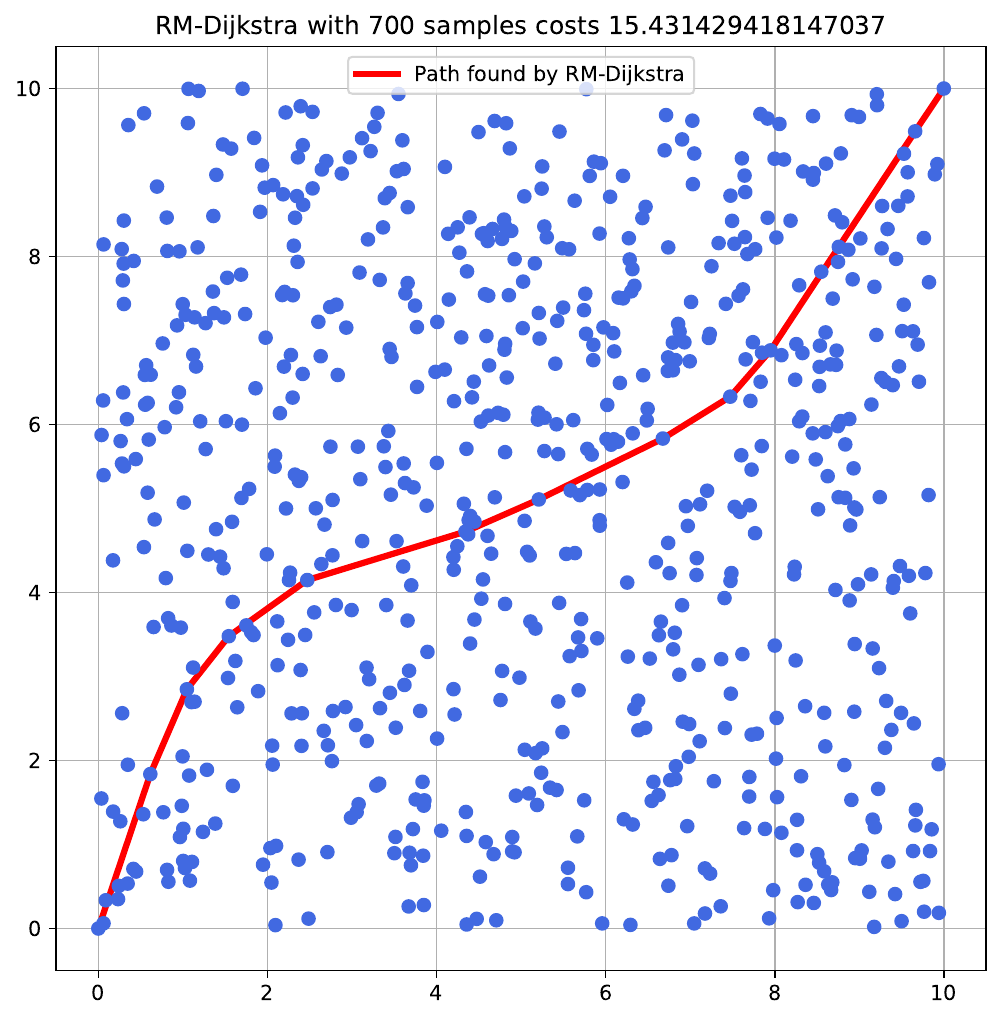}
      \label{3d3-2}
    }\\
    \subfloat[]{
      \includegraphics[width=0.5\linewidth]{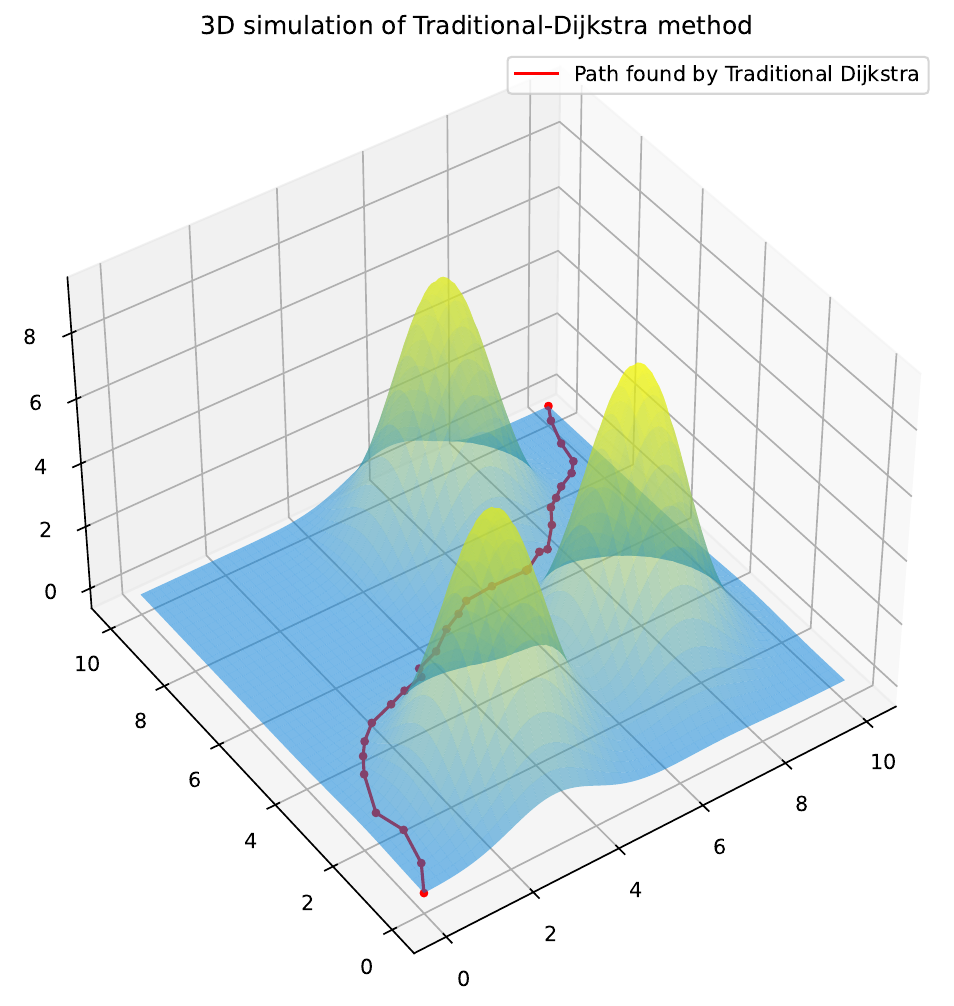}
      \label{3d3-3}
    }
    \subfloat[]{
      \includegraphics[width=0.5\linewidth]{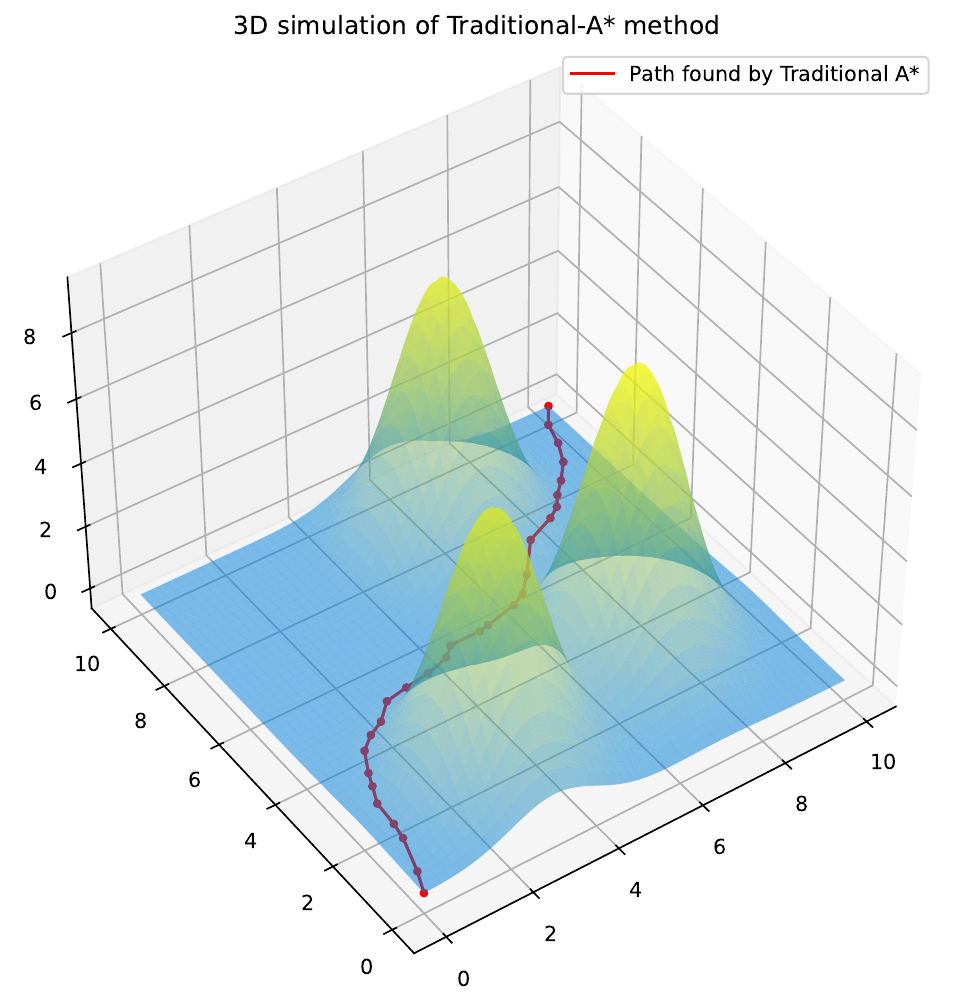}
      \label{3d3-4}
    }
   \caption{Simulation in 3-D scenario (three peaks). (b) is obtained by RM-Dijkstra taking 700 random sampling points. (a) is the preimage of the path in (b) under the projection map. (c) is obtained by original Dijkstra algorithm using Euclidean distance. (d) is obtained by original A* algorithm using Euclidean distance}
   \label{threepeaks}
\end{figure}

Next, we consider a more complex experimental scenario where we use the superimposed Gaussian function $x_3 = 8\cdot e^{ -\frac{1}{2}\left [(x_1-3)^2+(x_2-2)^2 \right ] } +  9\cdot e^{ -\frac{1}{2}\left [(x_1-7)^2+(x_2-3)^2 \right ] } +  8\cdot e^{ -\frac{1}{2}\left [(x_1-6)^2+(x_2-8)^2 \right ] } $ to simulate a surface with three peaks.
This time, the length of the path selected by RM-Dijkstra algorithm is 15.431429, as shown in Figure \ref{3d3-2} and \ref{3d3-1}.
This path, consisting of smooth curve segments, winds through the valleys between the three peaks, avoiding the peak areas with drastic changes in height, which is consistent with the intuitive requirement of the shortest path.
As comparative experiments, the length of the path retrieved by traditional Dijkstra algorithm (Figure \ref{3d3-3}) is 16.145669, and the length of the path retrieved by the traditional A* algorithm (Figure \ref{3d3-4}) is 16.087650.
The three paths follow a similar trajectory, but vary greatly in accuracy and smoothness.
However, when we use function $x_3 = 5\cdot e^{ -\frac{1}{2}\left [(x_1-3)^2+(x_2-2)^2 \right ] } +  5\cdot e^{ -\frac{1}{2}\left [(x_1-7)^2+(x_2-3)^2 \right ] } +  5\cdot e^{ -\frac{1}{2}\left [(x_1-3)^2+(x_2-7)^2 \right ] } +  5\cdot e^{ -\frac{1}{2}\left [(x_1-7)^2+(x_2-7)^2 \right ] } $ to simulate a more complex surface of four peaks (Figure \ref{fourpeaks}), we can see that the trend of the path retrieved by the traditional Dijkstra algorithm (Figure \ref{3d4-3}) has been greatly different from the trend of the path selected by the RM-Dijkstra algorithm (Figure \ref{3d4-1} and \ref{3d4-2}).
In the experimental scenario of the four-peak surface, traditional Dijkstra algorithm did not find a better way to reach the end point, but took a "detour".
This time, the path lengths retrieved by RM-Dijkstra algorithm, traditional Dijkstra algorithm, and traditional A* algorithm are 15.439869, 16.941380, and 17.027652, respectively.
Obviously, the accuracy advantage of RM-Dijkstra algorithm gradually increases with the increase of surface complexity.

The experimental results prove that RM-Dijkstra algorithm not only solves the optimal path planning problem of surface, but also has significant advantages over other traditional path planning algorithms in terms of path accuracy and smoothness when managing complex scenarios.

Solving the path planning problem by constructing a new Riemannian metric model is a challenging task, and related research in this area is limited. In this paper, we choose to model the raised peaks using a smooth normal distribution function. For cases involving non-smooth functions, we plan to explore them in future work, potentially utilizing smoothing techniques and other methods to address these challenges.

\begin{figure}
   \centering
   \subfloat[]{
    \includegraphics[width=0.5\linewidth]{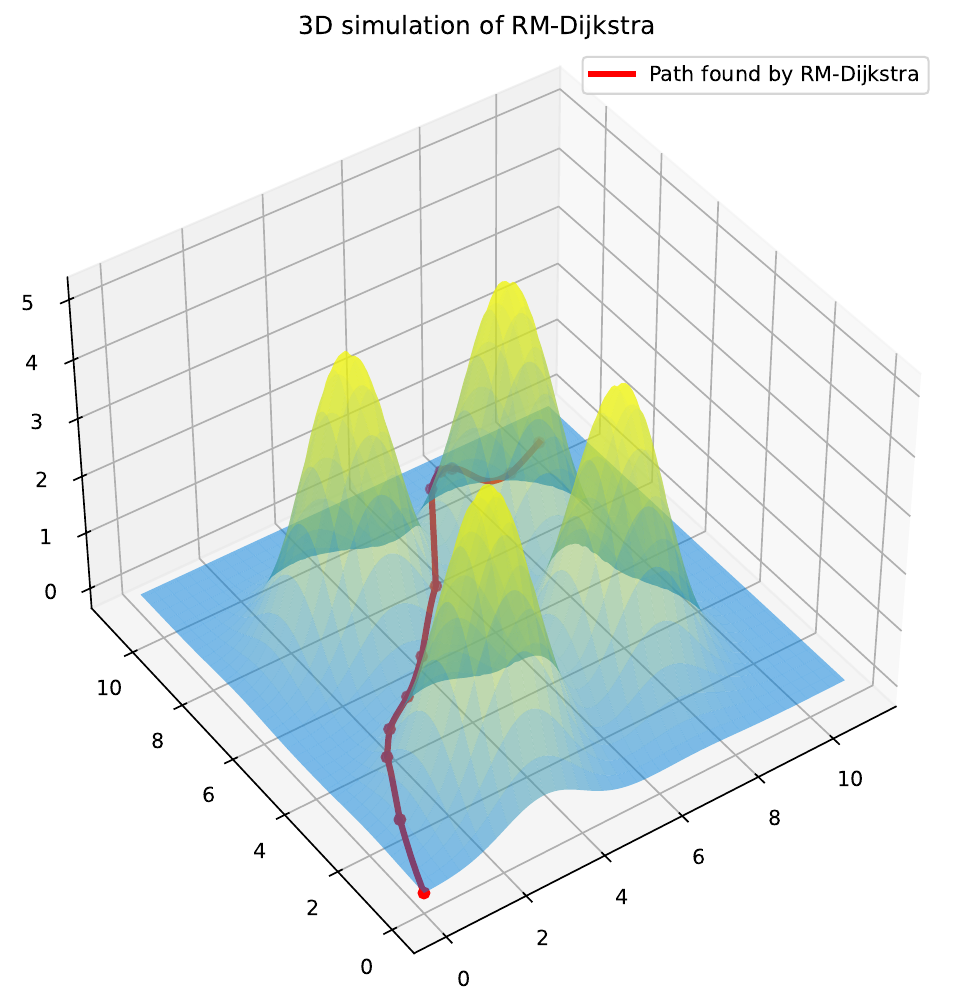}
    \label{3d4-1}
    }
   \subfloat[]{
      \includegraphics[width=0.47\linewidth]{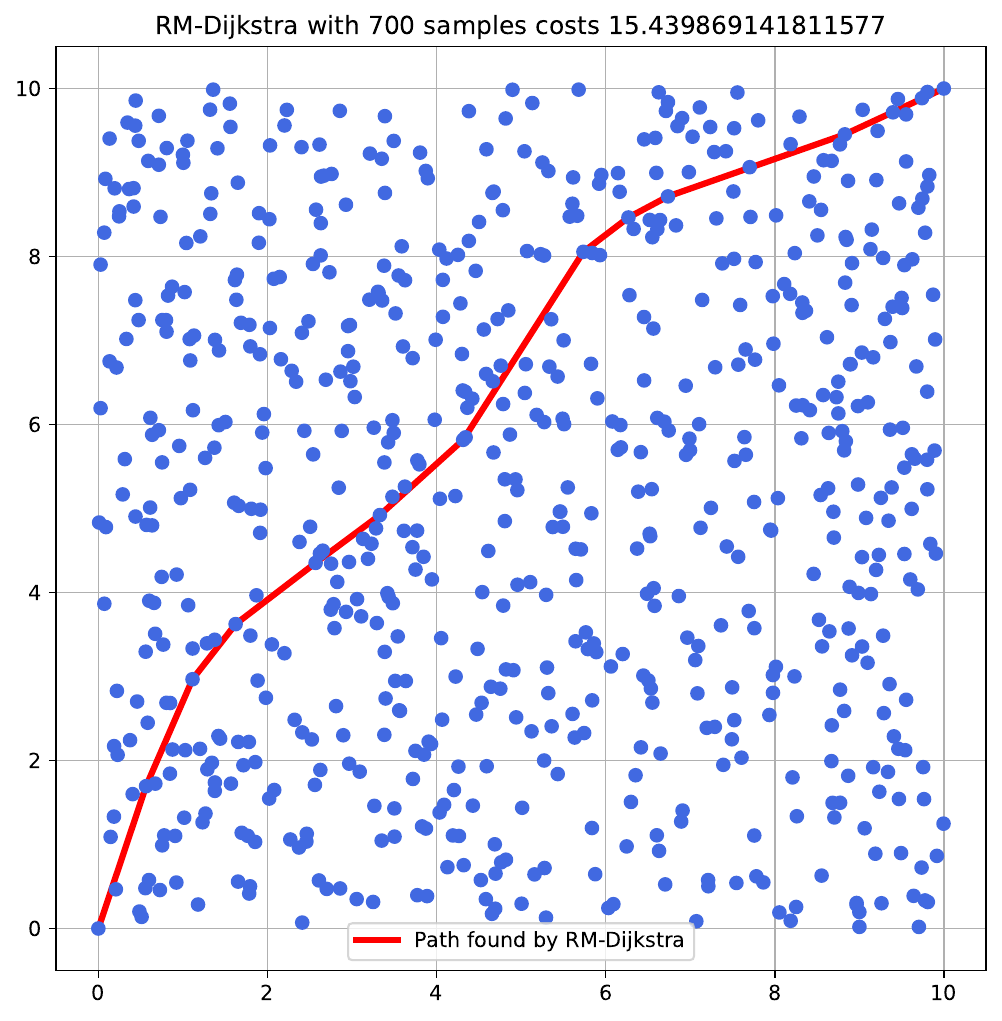}
      \label{3d4-2}
    }\\
    \subfloat[]{
      \includegraphics[width=0.5\linewidth]{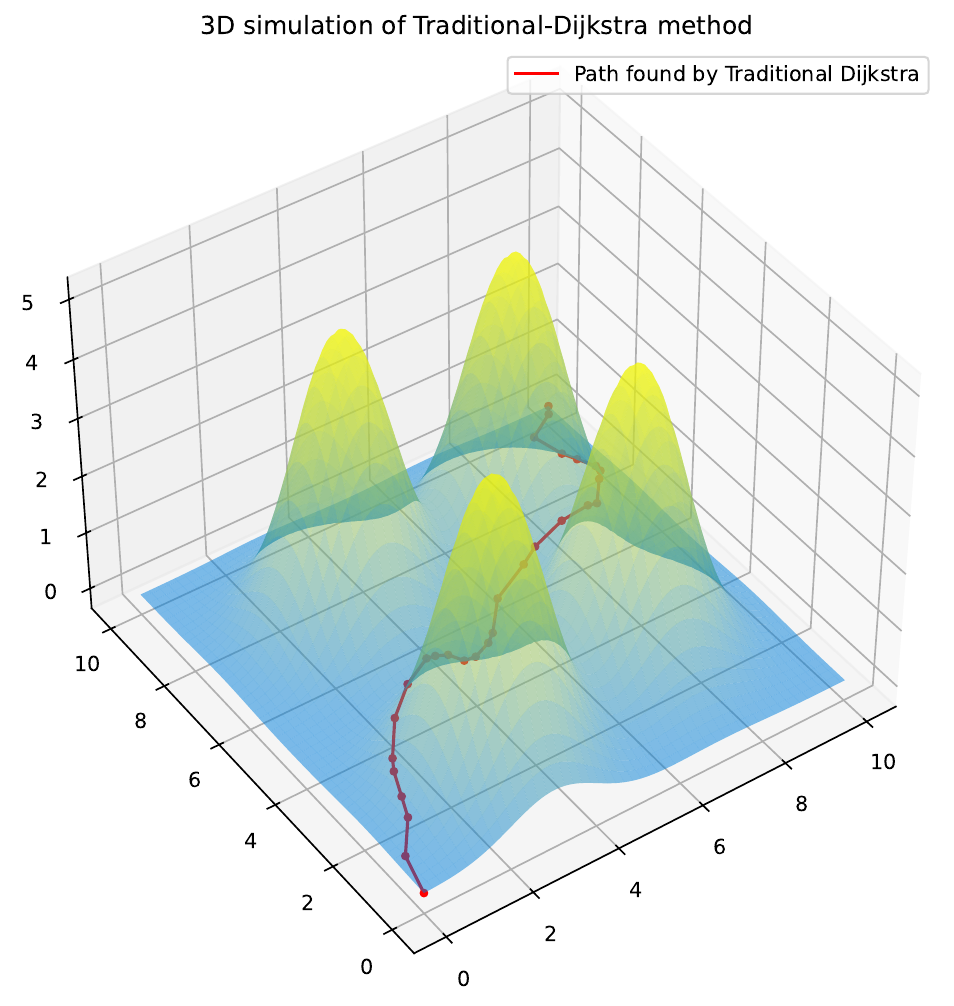}
      \label{3d4-3}
    }
    \subfloat[]{
      \includegraphics[width=0.5\linewidth]{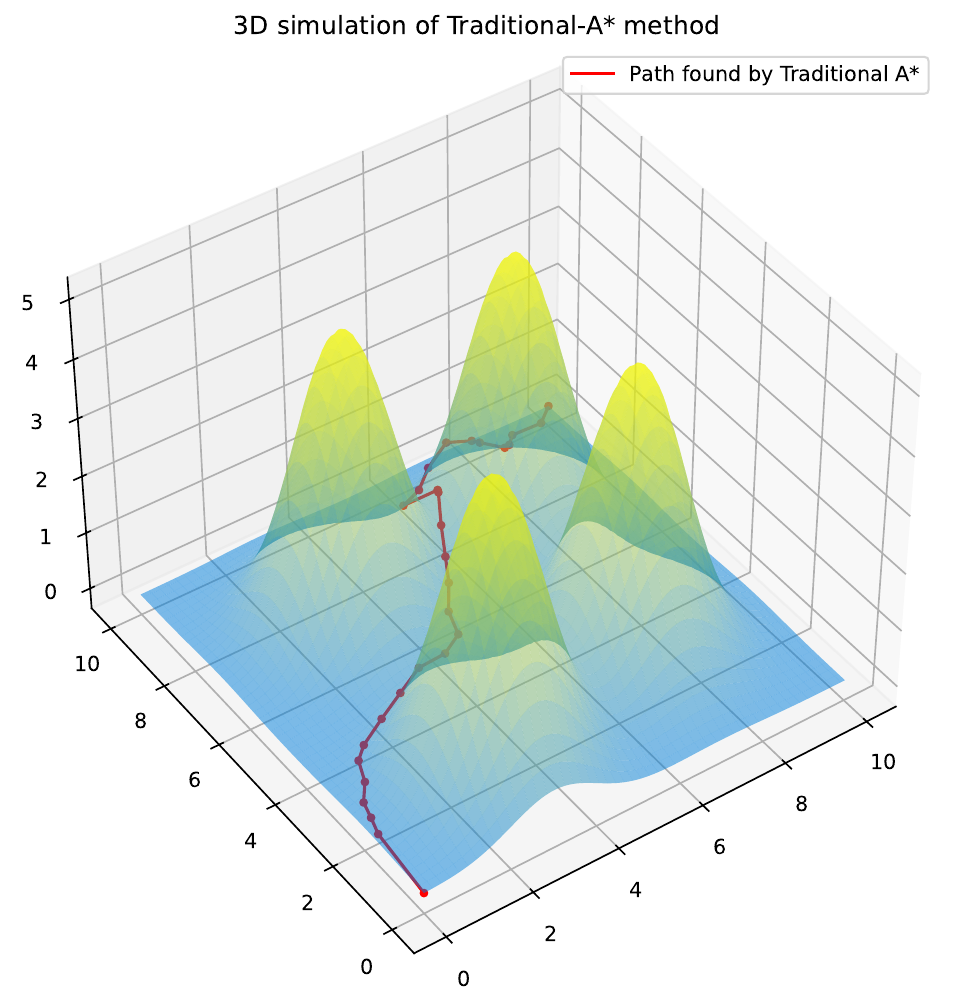}
      \label{3d4-4}
    }
   \caption{Simulation in 3-D scenario (four peaks). (b) is obtained by RM-Dijkstra taking 700 random sampling points. (a) is the preimage of the path in (b) under the projection map. (c) is obtained by original Dijkstra algorithm using Euclidean distance. (d) is obtained by original A* algorithm using Euclidean distance}
   \label{fourpeaks}
\end{figure}

\section{Conclusion}
\label{sec:Conclusion}
In this letter, an efficient optimal path planning framework is proposed for robot path planning on the 2D surfaces. By constructing a new Riemannian metric on the 2D projection plane, the surface optimal path planning problem is therefore transformed into a geometric problem on the 2D projection plane with new Riemannian metric. The projection map is an isometric embedding, which ensures that the length of a curve on surface is equal to the length of a new curve on its 2D projection plane in the new metric sense. Based on the above method, we propose a surface optimal path planning algorithm called RM-Dijkstra, which can accurately plan shortest smooth path from one point to target point on surfaces.
Through extensive simulation experiments, including scenarios with varying surface curvatures, we demonstrated the validity of the RM-Dijkstra algorithm. Comparative experiments with traditional path planning algorithms, such as the original Dijkstra and A* algorithms using Euclidean distance, showed that RM-Dijkstra significantly outperforms these methods in terms of path accuracy and piecewise smoothness, particularly in complex environments.
The results show the effectiveness of RM-Dijkstra in addressing the challenges of surface path planning problem and its potential for real-world applications where traditional algorithms may not perform as well. Future work will focus on further optimizing the algorithm for real-time applications and exploring its extension to more dynamic and complex robotic systems.

\vfill

\end{document}